\documentclass{article}

\usepackage{microtype}
\usepackage{graphicx}
\usepackage{subfigure}
\usepackage{booktabs} 
\usepackage{xspace}
\usepackage{amsmath,amssymb,amsfonts,amsthm,bbm}

\usepackage{natbib}
\renewcommand{\cite}[1]{\citep{#1}}

\usepackage{hyperref}
\usepackage[capitalize]{cleveref}

\DeclareMathOperator*{\argmax}{arg\,max}
\DeclareMathOperator*{\argmin}{arg\,min}

\newcommand{\e}[0]{\mathbb{E}\xspace}
\newcommand{\sign}[0]{\operatorname{sign}}
\newcommand{\indicator}[1]{\mathbbm{1}\{#1\}\xspace}
\newcommand{\abs}[1]{|#1|\xspace}
\newcommand{\norm}[1]{\|#1\|\xspace}
\newcommand{\sigmoid}[0]{\operatorname{sigmoid}}
\newcommand{\hx}[0]{h_\theta(X)\xspace}

\newcommand{\gx}[0]{g_{\theta}(X)\xspace}
\newcommand{\gxc}[2]{g_{#1}(#2)\xspace}

\newcommand{\regretj}[0]{\text{Regret}_J\xspace}
\newcommand{\regretl}[0]{\text{Regret}_L\xspace}
\newcommand{\asymregretl}[0]{\text{AR}_L\xspace}

\newcommand{\finitegmm}[0]{\textsc{FiniteGMM}\xspace}
\newcommand{\deepgmm}[0]{\textsc{ESPRM}\xspace}
\newcommand{\linearscenario}[0]{\textsc{Linear}\xspace}
\newcommand{\quadraticscenario}[0]{\textsc{Quadratic}\xspace}

\newcommand{\entropylearning}[0]{\textsc{ERM}\xspace}

\newcommand{\ie}{\emph{i.e.}}
\newcommand{\eg}{\emph{e.g.}}

\newcommand{\Eg}{\emph{E.g.}}

\newcommand{\ts}{\textstyle}

\newtheorem{lemma}{Lemma}
\newtheorem{assumption}{Assumption}
\newtheorem{theorem}{Theorem}

\PassOptionsToPackage{sort}{natbib}

\title{Efficient Policy Learning from Surrogate-Loss Classification Reductions}

\author{
Andrew Bennett\\
Cornell University\\
\texttt{awb222@cornell.edu}
\and
Nathan Kallus\\
Cornell University\\
\texttt{kallus@cornell.edu}
}

\newif\ifpreprint
\preprinttrue

\begin{document}

\maketitle

\begin{abstract}
Recent work on policy learning from observational data has highlighted the importance of efficient policy evaluation and has proposed reductions to weighted (cost-sensitive) classification. But, efficient policy evaluation need not yield efficient estimation of policy parameters. We consider the estimation problem given by a weighted surrogate-loss classification reduction of policy learning with any score function, either direct, inverse-propensity weighted, or doubly robust. We show that, under a correct specification assumption, the weighted classification formulation need not be efficient for policy parameters. We draw a contrast to actual (possibly weighted) binary classification, where correct specification implies a parametric model, while for policy learning it only implies a semiparametric model. In light of this, we instead propose an estimation approach based on generalized method of moments, which is efficient for the policy parameters. We propose a particular method based on recent developments on solving moment problems using neural networks and demonstrate the efficiency and regret benefits of this method empirically.
\end{abstract}

\section{Introduction}

Policy learning from observational data is an important but challenging problem because it requires reasoning about the effects of interventions not observed in the data. For example, if we wish to learn an improved policy for medical treatment assignment based on observational data from electronic health records, we must take care to consider potential confounding: since healthier patients who were already predisposed to positive outcomes were likely to have historically been assigned less invasive treatments, na\"ive approaches may incorrectly infer that a policy of always assigning less invasive treatments will obtain better outcomes.

Various recent work has recently tackled this problem, known as policy learning from observational (or, off-policy) data, by optimizing causally-grounded estimates of policy value such as inverse-propensity weighting (IPW), doubly robust (DR) estimates, or similar \citep{qian2011performance,beygelzimer2009offset,kitagawa2018should,swaminathan2015counterfactual,zhao2012estimating,zhou2017residual,jiang2019entropy,kallus2018policy,kallus2018balanced,kallus2017recursive}.
In particular, \citet{athey2017efficient,zhou2017residual}, among others, highlight the importance of using \emph{efficient} estimates of policy value as optimization objectives, \ie, having minimal asymptotic mean-squared error (MSE).
Examples of efficient estimators are direct modeling or IPW when outcome functions or propensities are sufficiently smooth \citep{hirano2003efficient,hahn1998role}, or 
DR leveraging cross-fitting \citep{chernozhukov2017double} in more general non-parametric settings.

Regardless of which of the three estimates one uses, the resulting optimization problem amounts to a difficult binary optimization problem. Therefore many of the above leverage a reduction of this problem to weighted classification (for two actions; cost-sensitive classification more generally) and leverage tractable convex formulations that use surrogate loss functions for the zero-one loss, such as, for example, hinge loss \citep[which yields a weighted SVM]{zhao2012estimating,zhou2017residual} and logistic loss \citep[which yields a weighted logistic regression]{jiang2019entropy}. The recently proposed entropy learning approach of \citet{jiang2019entropy} is particularly appealing, since the logistic regression-based surrogate loss is smooth and therefore allows for statistical inference on the estimated optimal parameters.

However, as we here emphasize, even if we use policy value estimates that are efficient, this \emph{does not} imply that we obtain efficient estimation/learning of the optimal policy itself, \emph{even} if the surrogate-loss model is well-specified.
For example, in the case of logistic loss,
we demonstrate that, although logistic regression is statistically efficient for actual binary classification when well-specified (as is well-known), in the case of policy learning via a weighted-classification reduction, well-specification only implies a \emph{semi}-parametric model and therefore minimizing the empirical average of loss is \emph{not} efficient in this case.

On the other hand, the implications of correct specification can be summarized as a conditional moment problem. Such problems are amenable to efficient solution using approaches based on the generalized method of moments \citep[GMM;][]{hansen1982large}. We demonstrate what an efficient such estimate would look like, in terms of the efficient instruments for our specific policy learning problem. We propose a particular implementation of solving our problem based on recent work on efficiently solving conditional moment problems using a reformulation of the efficient GMM solution as a smooth game optimization problem, which can be solved using adversarial training of neural networks \citep{bennett2019deep}. In addition, we prove some results relating the efficiency of optimal policy estimation to the asymptotic regret of the surrogate loss, and also prove that under correct specification the regret of the surrogate loss upper bounds the true regret of policy learning.

We demonstrate empirically over a wide range of scenarios that our methodology indeed leads to greater efficiency, with lower MSE in estimating the optimal policy parameter estimates under correct specification. Furthermore, we demonstrate that in practice, both \emph{with} and \emph{without} correct specification, our methodology tends to learn policies with \emph{lower regret}, particularly in the low-data regime.

\subsection{Setting and Assumptions}

Let $X$ denote the context of an individual, $T \in \{-1,1\}$ the treatment assigned to that individual, and $Y$ the resultant outcome. In addition let $Y(t)$ denote the counterfactual outcome that would have been obtained for the corresponding individual if, possibly contrary to fact, treatment $t$ had been assigned instead. We assume throughout that we have access to logged data consisting of $n$ iid observations, $\mathcal S_n=\{(X_i,T_i,Y_i):i\leq n\}$, of triplet $(X,T,Y)$ generated by some behavior policy.

We make standard causal assumptions of consistency and non-interference, which can be summarized by assuming that $Y = Y(T)$. Furthermore, as is standard in the above policy learning literature, we assume that $X$ encapsulates all possible confounders, that is, $Y(t)\bot T\mid X\ \forall t \in \{-1,1\}$, as would for example be guaranteed if the logging policy is a function of the observed individual context.

A policy $\pi$ denotes a mapping from individual context to treatment to be assigned. Concretely, given individual context $x$, let $\pi(x)\in\{-1,+1\}$ denote the treatment assigned by policy $\pi$ (we may also consider stochastic policies but since optimal policies are deterministic we focus on these).

Let 
\begin{align*}J(\pi) &= \e[Y(\pi(X))]-\frac12\e[Y(+1)+Y(-1)]\\&=\e[\pi(X)(Y(+1)-Y(-1))]\end{align*}
denote the expected value of following policy $\pi$, relative to complete randomization. Given the logged data and some policy class $\Pi$, our task is to learn an \emph{optimal policy} from the class, defined by $\pi^* \in \argmax_{\pi \in \Pi} J(\pi)$ (notice that offsetting by the complete randomization policy does not affect this optimization problem). In particular we consider policy classes where each policy $\pi$ is indexed by some utility function $g$ and is defined by $\pi(x) = \sign(g(x))$, where in turn the utility functions are parametrized by $\theta\in\Theta\subseteq\mathbb R^d$ as $\mathcal G = \{g_\theta : \theta \in \Theta\}$, so that $$\Pi=\{\sign(g_\theta(x)):\theta\in\Theta\}.$$
Correspondingly, we define $$J(\theta)=J(\sign(g_\theta(\cdot)))=\e[\sign(g_\theta(X))(Y(+1)-Y(-1))]$$
and $\theta^*\in\argmax_{\theta\in\Theta}J(\theta)$.
A prominent example is linear decision rules, where $g_\theta(x)=\theta^Tx$. Other examples include decision trees of bounded depth and neural networks.

\subsection{Efficiency}
\label{sec:efficiency}

We briefly review what it means to estimate the optimal policy parameters, $\theta^*$, efficiently. For simplicity, suppose that $\theta^*$ is unique.
A \emph{model} $\mathcal M$ is some set of distributions for the data-generating process (DGP), \ie, a set of probability distributions for the triplet $(X,T,Y)$.
 
A model is generally non-parametric in the sense that this set of distributions can be arbitrary, infinite, and infinite dimensional. 

Consider any learned policy parameters $\hat\theta$, that is, a function of the data $\mathcal S_n$ with values in $\Theta$. 
Roughly speaking, we say that $\hat\theta$ is \emph{regular} if, whenever the data is generated from $(X_i,T_i,Y_i)\sim p\in\mathcal M$, we have that $\sqrt{n}(\hat\theta-\theta^*)$ converges in distribution to some limit as $n\to\infty$ and this limit holds in a particular locally uniform sense in $\mathcal M$ (see \citealp[Chapter 25]{van2000asymptotic} for a precise definition).
Semiparametric efficiency theory (see ibid.) then establishes that there exists a covariance matrix $V$ such that for any cost function $c:\mathbb R^d\to\mathbb R$ for which the sublevel sets are $\{v:c(v)\leq c_0\}$ are convex, symmetric about the origin, and closed, we have that
\begin{equation}\label{eq:efficiency}
\liminf_{n\to\infty} \e[c(\sqrt{n}(\hat\theta-\theta^*))]
\geq \e_{v\sim\mathcal N(0,V)}[c(v)]
\end{equation}
for any estimator $\hat\theta$ that is regular in $\mathcal M$. An important example is MSE, given by $c(v)=\|v\|_2^2$.

\emph{Efficient} estimators are those for which \cref{eq:efficiency} holds with \emph{equality} for all such functions $c$, which, by the portmanteau lemma, would be implied if the estimator has the limiting law $\sqrt{n}(\hat\theta-\theta^*)\Rightarrow\mathcal N(0,V)$. Regular estimators is a very general class of estimators so the bound in \cref{eq:efficiency} is rather strong. So much so that, in fact, \cref{eq:efficiency} holds in a local asymptotic minimax sense for \emph{all} estimators (see ibid., Theorem 25.21).

Efficiency is important because, in observational data, we only have the data that we have and cannot experiment or simulate to generate more so we should use the data optimally. \Cref{eq:efficiency} relates to the efficiency of \emph{estimating} $\theta^*$. In \cref{sec:optimalregret} we also relate this to \emph{regret}.

\subsection{Related Work}

There has been a variety of past work on the problem of policy learning from observational data. Much of this work considers formulating the objective of policy learning as a weighted classification problem \citep{beygelzimer2009offset,dudik2011doubly}, and either minimizing the 0-1 loss directly using combinatorial optimization \citep{athey2017efficient,kitagawa2018should,zhou2018offline}, using smooth stochastic policies to obtain a nonconvex but smooth loss surface \citep{swaminathan2015counterfactual}, or replacing the 0-1 objective with a convex surrogate to be minimized instead \citep{zhao2012estimating,zhou2017residual,jiang2019entropy,beygelzimer2009offset,dudik2011doubly}. In addition there is work that extends some of the above approaches to the continuous action setting \citep{kallus2018policy,krishnamurthy2019contextual,chernozhukov2019semi}; our focus will be solely on binary actions. Of these methods the convex-surrogate approach has the advantage of computational tractability and, when the convex surrogate is smooth \citep[\eg][]{jiang2019entropy}, the ability to perform statistical inference on the optimal parameters. Our paper extends this work by investigating how to solve the smooth surrogate problem efficiently. Although much of this past work has used objective functions for learning based on statistically efficient estimates of policy value \citep{dudik2011doubly,athey2017efficient,zhou2018offline,chernozhukov2019semi}, to the best of our knowledge our paper is novel in investigating the efficient estimation of the optimal policy parameters themselves.

In addition there has been a variety of past work on solving conditional moment problems (see \citet{khosravi2019non,bennett2019deep} and citations therein). Our paper builds on this work as it reformulates the problem of policy learning as a conditional moment problem, which we propose to solve using optimally weighted GMM \citep{hansen1982large} and DeepGMM \citep{bennett2019deep}.

\section{The Surrogate-Loss Reduction and Its Fisher Consistency}

In this section, we present the surrogate-loss reduction of policy learning and the implications of correct specification.

\label{sec:entropy-learning}

Many policy learning methods start by recognizing
that the policy value can be re-written as
\begin{equation}\label{eq:policyscoring}
J(\theta)=\e[\psi \sign(g_\theta(X))]
\end{equation}
where $\psi$ is any of the following score variables, which all depend on observables:
\begin{align}
    \ts\psi_{\text{IPS}} &= \ts\frac{TY}{e_T(X)},\quad
    \psi_{\text{DM}} = \mu_1(X) - \mu_{-1}(X), \nonumber \\
    \ts\psi_{\text{DR}} &= \ts\psi_{\text{DM}} + \psi_{\text{IPS}} - \frac{T\mu_T(X)}{e_T(X)},\label{eq:psi}
\end{align}
where $e_t(x) = P(T=t \mid X=x)$ and $\mu_t(x) = \e[Y(t) \mid X=x]$. 
\Cref{eq:policyscoring} arises once we recognize that all of these satisfy $\e[\psi \mid X] = \e[Y(1) - Y(-1) \mid X]$.

Then we can approximate \cref{eq:policyscoring} using its empirical version:
\begin{equation}\label{eq:policyscoring_n}
\ts J_n(\theta)=\frac1n\sum_{i=1}^n\psi_i \sign(g_\theta(X_i)).
\end{equation}
In particular, \citet{athey2017efficient,kitagawa2018should,zhou2018offline} prove bounds of the form $\sup_{\theta\in\Theta}\abs{J_n(\theta)-J(\theta)}=O_p(1/\sqrt{n})$ given that the policy class has bounded complexity. This shows that optimizing $\hat\theta\in\argmax_{\theta\in\Theta}J_n(\theta)$ provides near-optimal solutions in the original policy learning problem, since $J(\theta^*)-J(\hat\theta)\leq J(\theta^*)-J(\hat\theta)+J_n(\hat\theta)-J_n(\theta^*)\leq 2\sup_{\theta\in\Theta}\abs{J_n(\theta)-J(\theta)}$.
Given that in practice the nuisance functions $e_t$ and $\mu_t$ are estimated from data, we denote the corresponding score variable when such estimates are plugged in as $\hat\psi$ to differentiate it from the variable $\psi$ that uses the true nuisance functions. We correspondingly let $\hat J_n(\theta)=\frac1n\sum_{i=1}^n\hat\psi_i \sign(g_\theta(X_i))$. When $\hat J_n(\theta)$ is efficient for $J(\theta)$ one can generally additionally prove that $\sup_{\theta\in\Theta}\abs{\hat J_n(\theta)-J_n(\theta)}=o_p(1/\sqrt{n})$.

Given the non-convexity and non-smoothness of the empirical objective function \cref{eq:policyscoring_n} it is not necessarily clear how to actually optimize it, however. Many works \citep{jiang2019entropy,zhao2012estimating,beygelzimer2009offset} recognize that this optimization problem is actually equivalent to weighted binary classification (in our two-action case), since $\psi_i\sign(g_\theta(X_i))=\abs{\psi_i}(1-2\mathbb I_{\sign(g_\theta(X_i)\neq\psi_i})$, so any classification algorithm that accepts instance weights can perhaps be used to address \cref{eq:policyscoring_n}.
Specifically, many classification algorithms take the form of minimizing a \emph{convex surrogate loss}:
\begin{equation}\label{eq:sampleL}
\ts    L_n(\theta) = \frac{1}{n} \sum_{i=1}^n \abs{\psi_i} l(g_\theta(X_i), \sign(\psi_i)),
\end{equation}
where $l(g,s)$ acts as a surrogate for the zero-one loss $\mathbb I_{\sign(g_\theta(X_i)\neq\psi_i}$. Analogous to above, we let $L(\theta) = \e[\abs{\psi}l(g_\theta(X),\sign(\psi))]$ denote the population version of this loss. For classification, \citet{bartlett2006convexity} studies which losses are appropriate surrogates, \ie, are classification-calibrated.
The population version of the surrogate loss, which $L_n(\theta)$ is approximating, is 
\begin{equation}\label{eq:popL}L(\theta)=\e[\abs{\psi} l(g_\theta(X), \sign(\psi))].\end{equation}

Following \citet{jiang2019entropy}, we will focus on the logistic (or, logit-cross-entropy) loss function and define $l(g,s)$ everywhere as:\footnote{All of our results actually extend to any twice-differentiable classification-calibrated loss. Logistic is the most prominent such loss.}
\begin{equation*}
    l(g, s) = 2 \log(1 + \exp(g)) - (s + 1) g.
\end{equation*}
This loss is clearly smooth, and $l(g_\theta(x),s)$ is also convex in $\theta$ as long as $g_\theta(x)$ is convex in $\theta$ for each $x$. The loss is also classification-calibrated, which immediately yields the following, given an additional regularity assumption:
\begin{assumption}
\label{asm:psi-bounded}
$\e[\psi \mid X] = \e[Y(1) - Y(-1) \mid X]$, and $\e[\abs{\psi}] < \infty$.
\end{assumption}
\begin{theorem}[Fisher Consistency Under Correct Specification]
\label{thm:correct-specification}
Suppose the policy class $\Pi$ is correctly specified for the surrogate loss in the sense that
\begin{equation}\label{eq:specification}\mathcal G\cap\left({\argmin_{g~\text{unconstrained}}\e[\abs{\psi} l(g(X), \sign(\psi))]}\right)\neq\varnothing.\end{equation}
Then given \cref{asm:psi-bounded}, any minimizer of the surrogate-loss risk is an optimal policy:
\begin{align*}
&J(\theta^*)=\max_{\pi~\text{unconstrained}}J(\pi)\\
&\text{for all}~\theta^*\in\argmin_{\theta\in\Theta} L(\theta).
\end{align*}
\end{theorem}

\Cref{thm:correct-specification} establishes that, under correct specification, if we minimize the population surrogate loss, $L(\theta)$, then we obtain the optimal policy. Therefore, a natural strategy for policy learning would be to directly minimize the empirical loss $L_n(\theta)$, as was done by the above. Although the above arguments indicate that this approach would be computationally tractable, and also consistent under mild regularity conditions that ensure that optimizers of $L_n(\theta)$ would converge to optimizers of $L(\theta)$, it is not clear that it is statistically efficient, even if we use an efficient score variable for policy value estimation.

\section{The Conditional-Moment Reformulation of the Surrogate-Loss Reduction}\label{sec:reform}

In this section we establish a new interpretation of the surrogate-loss reduction as a conditional moment problem and we discuss the implications of this in terms of the model implied by correct specification. This will enable us to conduct efficiency analysis and to design algorithms with improved efficiency in the next section.

\subsection{The Conditional Moment Problem}

To make progress toward a characterization of efficiency under correct specification, we next establish an equivalent formulation of optimizing the population surrogate loss under correct specification as a conditional moment problem.

Define the derivative of $l$ with respect to $g$:
\begin{equation*}
    l'(g, s) = 2 \sigma(g) - (s + 1),
\end{equation*}
where $\sigma(g)=\exp(g)/(1+\exp(g))$ is the logistic function.

\begin{theorem}[Conditional Moment Problem Under Correct Specification]\label{thm:conditoinalmomentproblem}
Suppose \cref{asm:psi-bounded} holds and the policy class $\Pi$ is correctly specified for the surrogate loss in the sense that \cref{eq:specification} holds.
Define
$$
m(X;\theta)=\e \left[ 
\abs{\psi}  l'(g_\theta(x), \sign(\psi))
\mid X \right].
$$
Then we have that
\begin{align}
&\theta^*\in\argmin_{\theta\in\Theta}L(\theta)\notag
\\&\iff m(X;\theta^*)= 0
~~\text{almost surely}.\label{eq:conditional-moment}
\end{align}
\end{theorem}

\Cref{thm:conditoinalmomentproblem} arises straightforwardly from the observation that, under correct specification, $g_\theta(x)$ minimizes $\e[\abs{\psi_i} l(g_\theta(X), \sign(\psi))\mid X=x]$ for almost every $x$. Using smoothness and convexity, this latter observation is restated using first-order optimality conditions. Dominated convergence theorem allows us to exchange differentiation and expectation and we obtain the result.
\Cref{thm:conditoinalmomentproblem} provides an alternative characterization of $\theta^*$ as solving a \emph{conditional moment problem}.

Notice that \cref{eq:conditional-moment} is \emph{equivalent} to the statement that, for any square integrable function $f$ of $X$, we have the moment restriction
\begin{equation}
\label{eq:moment-restriction}
\e \left[ m(X;\theta) f(X) \right] = 0. 
\end{equation}
This alternative characterization makes the problem amenable to efficiency analysis.

Notice that by first-order optimality, if $\theta^*\in\operatorname{interior}(\Theta)$, optimizing $L(\theta)$ in \cref{eq:popL} exactly corresponds to solving the set of $d$ moment equations given by 
$\e[m(X;\theta)\hx]=0$.
Similarly, optimizing the empirical loss $L_n(\theta)$ in \cref{eq:sampleL} corresponds to solving these $d$ equations with population averages ($\e$) replaced with empirical sample averages.

However, \cref{eq:moment-restriction} gives a much broader set of equations. Leveraging this fact will be crucial to achieving efficiency. 
Indeed, it is well-known that even if a small number of moment equations are sufficient to identify a parameter (\eg, in the above, the $d$ equations identify $\theta^*$ via first-order optimality), taking into consideration additional moment equations that are known to hold can increase efficiency in semiparametric settings \citep{carrasco2014asymptotic}.

\subsection{The Semiparametric Model Implied by Specification}

In order to reason about efficiency, we need to reason about the model implied by \cref{eq:conditional-moment}. To do so, we first establish the following lemma:
\begin{lemma}
\label{lemma:semimodel}
    Assume \cref{asm:psi-bounded}. Then given a policy class $\Pi$, the model of DGPs (distributions on $(X,T,Y)$) where $\Pi$ is correctly specified for the surrogate loss (in the sense of \cref{eq:specification}) is given by all distributions on $(X,T,Y)$ for which there exists $\theta^* \in \Theta$ satisfying
    \begin{equation}\label{eq:semimodel}
        \frac{P(\psi > 0 \mid X)}{\e \left[\abs{\psi} \mid X \right]} 
        =\sigma(g_{\theta^*}(X))~~\text{almost surely}.
    \end{equation}
\end{lemma}
This model is generally a \emph{semiparametric} model. That is, while \cref{eq:semimodel} is a parametric restriction on the function ${P(\psi > 0 \mid X=x)}/{\e \left[\abs{\psi} \mid X=x \right]}$, the set of corresponding distributions on $(X,T,Y)$ that satisfy this restriction is still infinite-dimensional and non-parametric.

\subsection{Comparison with Logistic Regression for Classification}

One question the reader might have at this point is why an approach different than empirical loss minimization is necessary for efficiency, given that the surrogate loss formulation seems mathematically identical to binary classification using logistic regression, which is known to be 
efficient.\footnote{This is because logistic regression performs maximum likelihood estimation (MLE), which is statistically efficient for well-specified parametric models.} The difference between the problems is that for actual classification we have that $\psi$ is a binary class label, \ie, $\psi \in \{-1, 1\}$. If we assume the policy class is well-specified and $\psi \in \{-1, 1\}$, the characterization of our semiparametric model from \cref{lemma:semimodel} reduces to
\begin{equation*}
    P(\psi = 1 \mid X)=\sigma(g_{\theta^*}(X)),
\end{equation*}
which implies that our model is \emph{parametric}, since the choice of $\theta^*$ now fully characterizes the distribution of the label $\psi$ given $X$. \Eg, usually for logistic regression we let $g_\theta(x)=\theta^Tx$ so that the above says that the logit of $P(\psi = 1 \mid X)$ is linear. Therefore, performing logistic regression corresponds to MLE for this parametric model, which is efficient. 

However in our general setting this is \emph{not} the case and there is a non-trivial nuisance space, since an infinite-dimensional space of conditional distributions for $\psi$ given $X=x$ could result in the \emph{same} function $P(\psi > 0 \mid X=x) / \e[\abs{\psi} \mid X=x]$. This suggests that we may need to be more careful in order to obtain efficiency and that there may exist estimators that are more efficient than empirical loss minimization.

\section{Efficient Policy Learning Reductions}

In this section we propose some efficient methods for policy learning based on the above conditional-moment formulation. In addition, we provide some analysis of these methods in terms of efficiency and regret.

\subsection{FiniteGMM Policy Learner}

We begin by proposing an approach based on using multi-step GMM to solve the conditional moment problem, which we will call \finitegmm. This approach works by optimally enforcing for the moment conditions given by \cref{eq:moment-restriction} for a finite collection of critic functions $\mathcal F =\{f_1,\ldots,f_k\}$. Specifically, given some initial estimate $\tilde \theta_n$ of $\theta^*$, define: 
\begin{align*}
    m(\theta)_j &=\ts \frac1n\sum_{i=1}^n\abs{\hat\psi_i} l'(g_\theta(X_i), \sign(\hat\psi_i)) f_j(X_i) \\
    C(\tilde \theta_n)_{jk} &=\ts \frac1n\sum_{i=1}^n\hat\psi_i^2 l'(\gxc{\tilde\theta_n}{X_i}, \sign(\hat\psi_i))^2 f_j(X_i) f_k(X_i) \\
    O(\theta;\tilde\theta_n) &= m(\theta)^T C(\tilde \theta_n) m(\theta).
\end{align*}
We then estimate $\theta$ by $\hat \theta_n = \argmin_{\theta} O(\theta;\tilde\theta_n)$. We can repeat this multiple times, plugging in $\hat\theta_n$ as $\tilde\theta_n$ and resolving.

An important issue with this estimator, however, is how to choose the critic functions. Standard GMM theory requires that the $k$ moment conditions are sufficient to identify $\theta^*$. And even then, the above is only the most efficient among estimators of the form $\argmin_{\theta} \norm{(m_1(\theta),\ldots,m_k(\theta)}$ for any norm $\norm{\cdot}$, but there may still be more efficient choices of critic functions.

\subsection{The Efficient Instruments for Policy Learning}
One nice result from the theory of conditional moment problems is the existence of a finite set of critic functions ensuring efficiency in the sense of \cref{sec:efficiency}. Define:
\begin{align*}
\label{eq:efficient-instruments}
    \Omega(x) &= \e[ \psi^2 l'(\gx, \sign(\psi))^2 \mid X=x ] \nonumber \\
    h_{\theta^*}(x) &= \nabla_{\theta} g_{\theta}(x)\mid_{\theta=\theta^*} \nonumber \\
    D(x) &= \e[ \nabla_{\theta} ( \abs{\psi} l'(\gx, \sign(\psi)) )\mid_{\theta=\theta^*}\mid X=x ] \nonumber \\
    &= \e[ \abs{\psi} l''(g_{\theta^*}(X), \sign(\psi)) h_{\theta^*}(x) \mid X=x]\nonumber \\
    f_i^*(x) &= \frac{D(x)_i}{\Omega(x)}.
\end{align*}
We call $\mathcal F^* = \{f_1^*,\ldots,f_d^*\}$ the \emph{efficient instruments}, and as long as the span of $\mathcal F$ contains these instruments then \finitegmm is guaranteed efficiency \citep{newey1993efficient}.

Given this, one approach would be to let $\mathcal F$ be flexible with the hope of approximately containing $\mathcal F^*$.
Letting, for example, $\mathcal F$ be the first $k(n)$ functions in a basis for $L_2$ such as a polynomial basis and letting $k(n)\to\infty$ can be shown to be efficient under certain conditions \citep{newey1993efficient}. This, however, can perform very badly in practice, especially with any reasonable amount of features.
Ideally, we would instead be able to make use of modern machine learning methods and approximate $\mathcal F^*$ using some flexible function class such as neural networks rather than defining a finite set of basis functions.

\subsection{ESPRM Policy Learner}
\label{sec:deepgmm}

Motivated by the above concerns, we now present our proposed approach: \deepgmm (efficient surrogate policy risk minimization).
This is based on the extension of \citet{bennett2019deep} to our conditional moment problem.
In the setting of instrumental variable regression,
\citet{bennett2019deep} proposes an adversarial reformulation of optimally-weighted GMM, which allows us to consider critic functions given by flexible classes such as neural networks. Then if this class provides a good approximation for the efficient instruments, this approach should be approximately efficient. 

Specifically, we define:
\begin{align*}
    u(X,\psi;\theta,f) &= \abs{\psi} l'(\gx, \sign(\psi)) f(X) \\
    U(\theta, f; \tilde\theta) &=\ts \frac1n\sum_{i=1}^nu(X_i,\hat\psi_i;\theta,f)] \\
    &\qquad\ts - \frac1{4n}\sum_{i=1}^nu(X_i,\hat\psi_i;\tilde\theta_n,f)^2,
\end{align*}
where as above $\tilde\theta_n$ is some initial consistent estimate of $\theta^*$. Then following \citet{bennett2019deep}, the \deepgmm estimator is defined as
\begin{equation*}\ts
    \hat\theta^{\deepgmm} = \argmin_\theta \sup_{f \in \mathcal G} U(\theta,f;\tilde\theta),
\end{equation*}
where $\mathcal F$ is our flexible function class (henceforth assumed to be a class of neural networks).
It remains to describe how this adversarial game is to be solved, and how to define $\tilde\theta_n$. As in \citet{bennett2019deep} we optimize the objective by performing alternating first-order optimization steps using the OAdam algorithm \citep{daskalakis2017training}, which was designed for solving smooth game problems such as generative adversarial networks (GANs). In addition, we continuously update $\tilde \theta_n$ during optimization, where at each step of alternating first order optimization we set $\tilde \theta_n$ equal to the previous iterate of $\hat \theta_n$.

\subsection{Efficient Learning implies Optimal Regret}\label{sec:optimalregret}

Finally we prove that efficiency not only ensures minimal MSE in estimating $\theta^*$ but also implies regret bounds. Let
\begin{align*}
    \regretj(\theta) &= \argmax_{\pi~\text{unconstrained}}J(\pi) - J(\theta) \\
    \regretl(\theta) &= L(\theta) - \inf_{\theta \in \Theta}L(\theta).
\end{align*}
\begin{theorem}[Regret Upper Bound]
\label{thm:regret-bound}
Suppose \cref{asm:psi-bounded} holds and that the policy class $\Pi$ is correctly specified for the surrogate loss in the sense that \cref{eq:specification} holds.
Then, for any $\theta \in \Theta$ we have:
\begin{equation*}
    \regretj(\theta) \leq \regretl(\theta).
\end{equation*}
\end{theorem}
This theorem implies that the regret of a policy is upper-bounded by the excess risk of the surrogate loss. Next, we make the following regularity assumption about the loss $L$:
\begin{assumption}[Well Behaved Loss]
\label{asm:regular-loss}
$L$ has a unique minimizer $\theta^*$ in the interior of $\Theta$, and the Hessian $H(\theta^*)$ of $L$ at $\theta^*$ is positive definite.
\end{assumption}
Given this assumption, a Taylor's theorem expansion yields $\regretl(\hat\theta_n) = (\hat\theta_n - \theta^*)^T H(\theta^*) (\hat\theta_n - \theta^*) + o(\norm{\hat\theta_n - \theta^*}^2)$. 
For for any regular estimator $\hat\theta_n$,
we can also define the \emph{asymptotic regret} $\asymregretl$ as the limiting distribution:
\begin{equation*}
    n \regretl(\hat\theta_n) \rightarrow_d \asymregretl(\hat\theta_n),
\end{equation*}
which exists since regularity implies that $\sqrt{n}(\hat\theta_n - \theta^*)$ has a limiting distribution. 
Given this we can prove the following optimality result of our efficient estimators in terms of asymptotic regret:

\begin{theorem}[Optimal Asymptotic Regret]
\label{thm:optimal-regret}
Given \cref{asm:regular-loss} and any non-negative, non-decreasing $\phi$, we define the risk $R_{\phi}(\hat\theta_n) = \e[\phi(\asymregretl(\hat\theta_n))]$. Given this, there exists a risk bound $B_{\phi}$ such that $R_{\phi}(\hat\theta_n) \geq B_{\phi}$ for every regular $\hat\theta_n$, with equality if $\hat\theta_n$ is semi-parametrically efficient.
\end{theorem}

Together with \cref{thm:regret-bound}, this means that both the actual regret ($\regretj$) and the surrogate regret ($\regretl$) of policies given by efficient estimators $\hat\theta$ are $O_p(1/n)$, and the surrogate regret has an optimal constant.

\section{Experiments}
\label{sec:experiments}

\begin{figure*}
\begin{center}
\ifpreprint
\begin{tabular}{c}
     \includegraphics[width=0.95\textwidth]{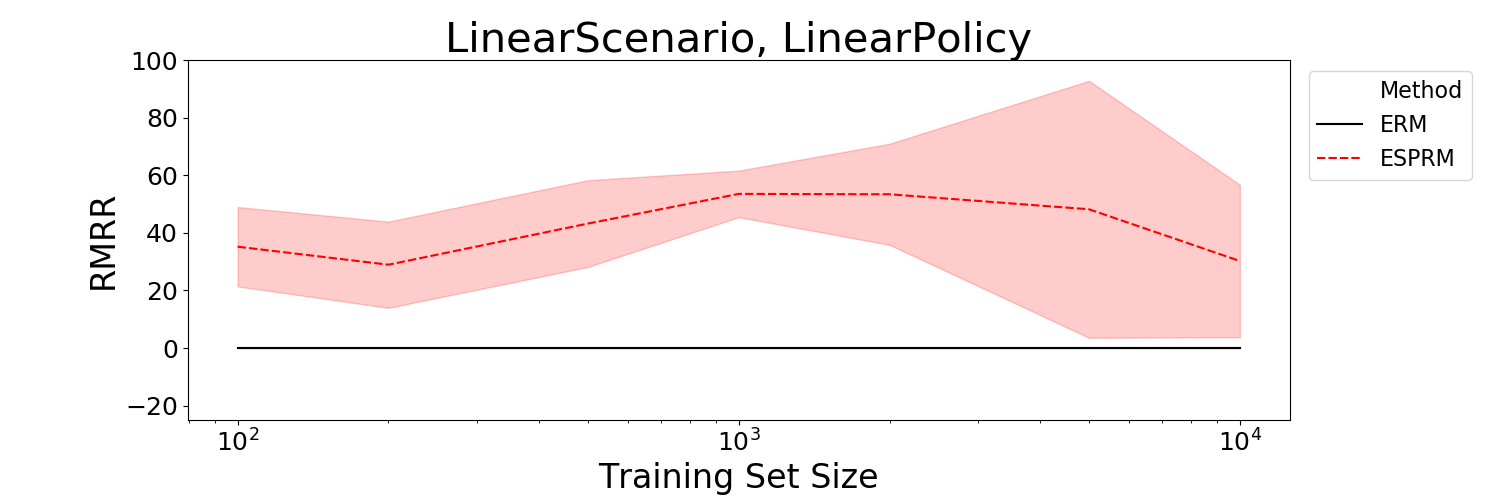} \\ 
     \includegraphics[width=0.95\textwidth]{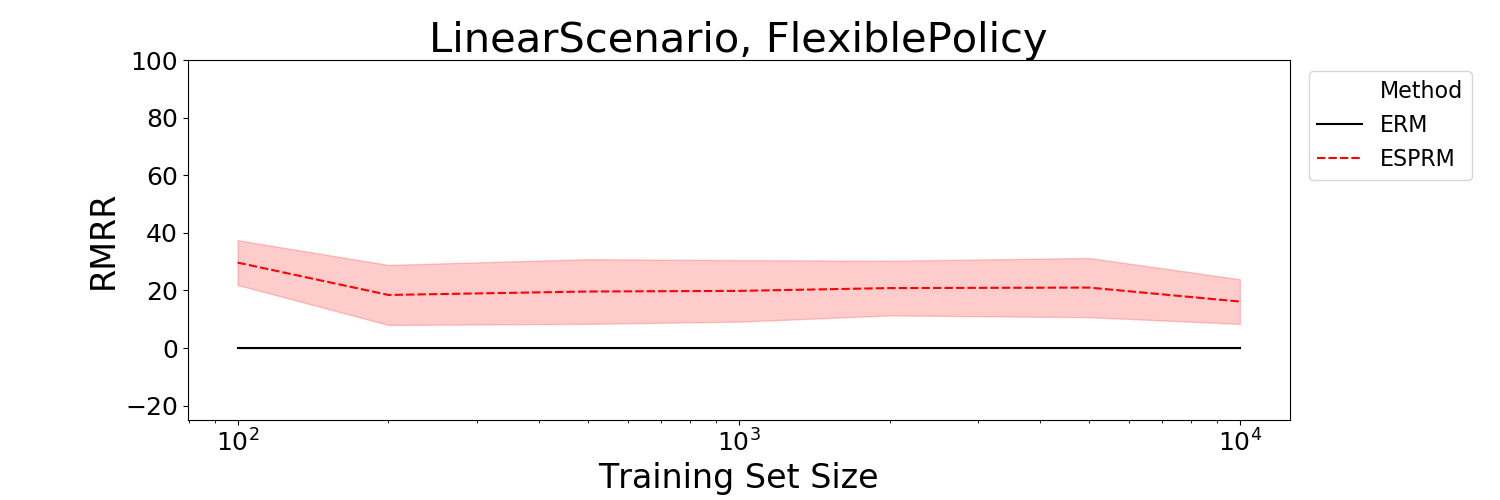} \\ 
     \includegraphics[width=0.95\textwidth]{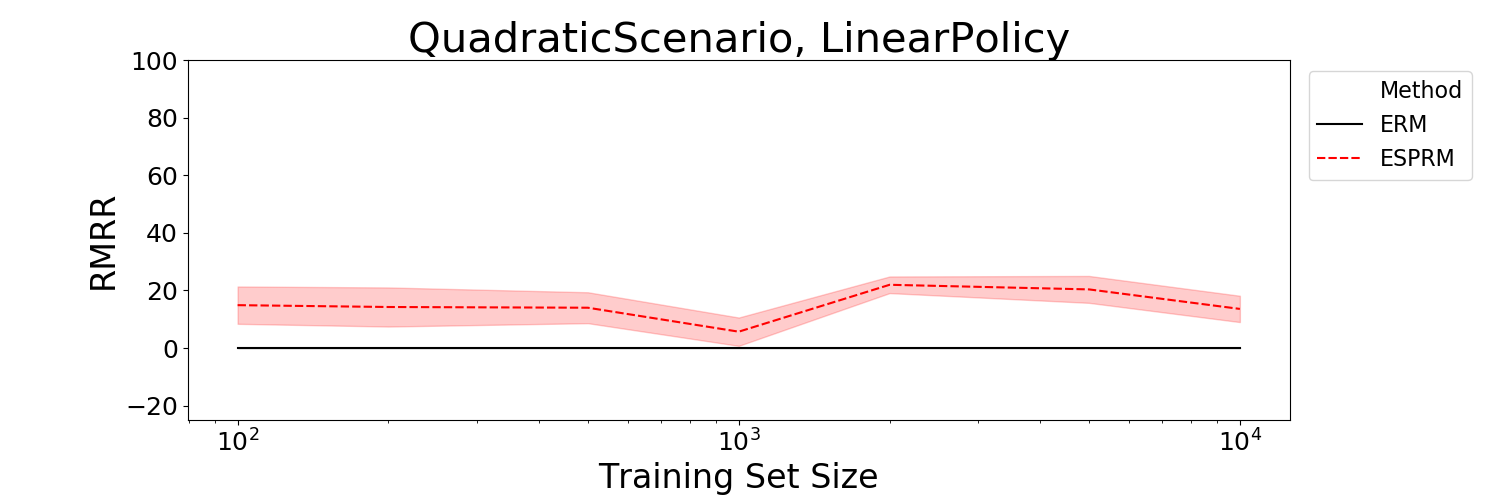} \\
     \includegraphics[width=0.95\textwidth]{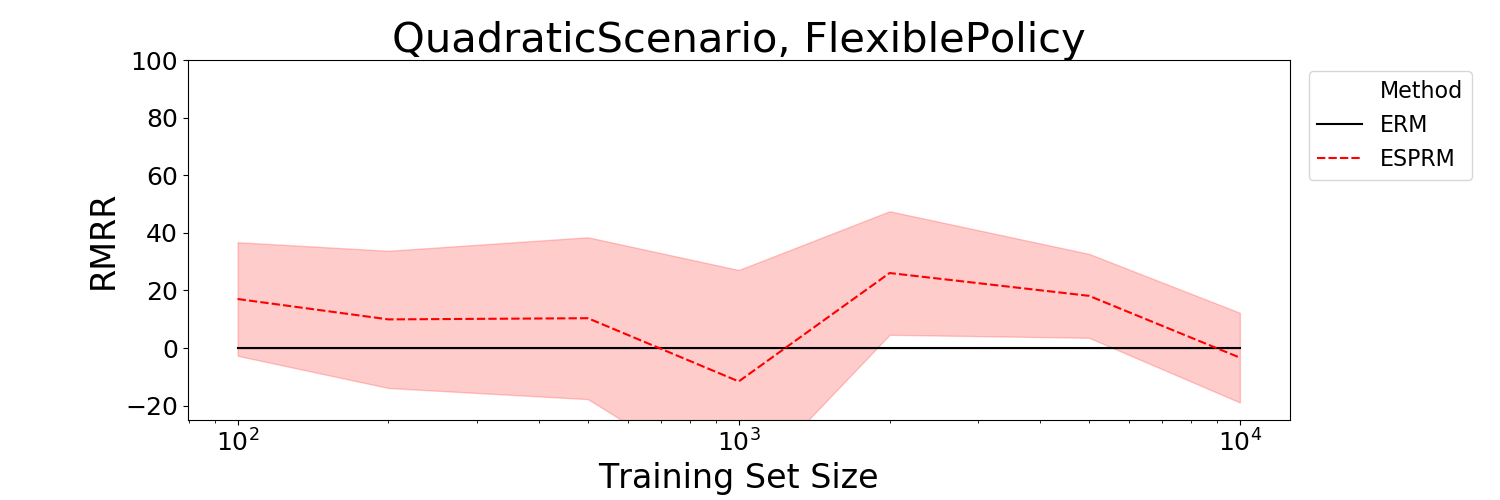}
\end{tabular}
\else
\begin{tabular}{cc}
     \includegraphics[width=0.45\textwidth]{"figs/simple_linear"} & \includegraphics[width=0.45\textwidth]{"figs/simple_flexible"} \\ 
     \includegraphics[width=0.45\textwidth]{"figs/quadratic_linear"} & \includegraphics[width=0.45\textwidth]{"figs/quadratic_flexible"}
\end{tabular}
\fi
\end{center}
\caption{Difference in performance between \deepgmm and \entropylearning. We plot RMRR against number of training examples for each combination of policy class and scenario kind. All shaded regions are 95\% confidence intervals constructed from bootstrapping.}
\label{fig:regret-plots}
\end{figure*}

\subsection{Synthetic Scenarios}

First we investigate the performance of our algorithms on a variety of synthetic scenarios. In all these scenarios $X$ is 2-dimensional, and $X$ and $Y(t) - \mu_t(X)$ are standard Gaussian distributed for each $t$; the scenarios only differ in the functions $\mu_t$ and $e_t$. In none of the scenarios is our policy class \emph{actually} well-specified in the sense of \cref{eq:specification}.

We consider the following kinds of synthetic scenarios:
\begin{itemize}
    \item \linearscenario: $\mu_t(x) = a_t^T x + a_{t0}$ and $e_1(x) = \sigmoid(b^Tx + b_0)$ for some vectors $a_{-1}$, $a_1$, $b$.
    \item \quadraticscenario: $\mu_t(x) = x^T A_t x + a_t^T x + a_{t0}$ and $e_1(x) = \sigmoid(x^T B x + b^Tx + b_0)$ for some symmetric matrices $A$ and $B$, and vectors $a_{-1}$, $a_1$, $b$.
\end{itemize}

In addition we experiment with the following policy classes: a \emph{linear} policy class, where $g_\theta(x) = \theta^Tx + \theta_0$, and a \emph{flexible} policy class where $g_\theta(x)$ is given by a fully-connected neural network with a single hidden layer of size 50, and leaky ReLU activations.

In all cases we use the surrogate loss method of \citet{jiang2019entropy} described in \cref{sec:entropy-learning} as a benchmark, which we henceforth refer to as \entropylearning. We note that although in the prior work they used $\hat \psi_{\text{IPS}}$, we instead use $\hat\psi_{\text{DR}}$, both because it is theoretically better grounded \citep{athey2017efficient,zhou2017residual} and we found that it gives stronger results for all methods.
For our \deepgmm method we let $\mathcal F$ be the same neural network function class as for flexible policies, and perform alternating first-order optimization as described in \cref{sec:deepgmm} for a fixed number of epochs.

For all methods, except where otherwise specified, we use the $\hat\psi_{DR}$ weights described in \cref{eq:psi}, with nuisance functions fit using correctly specified linear regression or logistic regression algorithms on a separately sampled tuning dataset of the same size as the training dataset.\footnote{By correctly specified we mean that for \linearscenario we fit using linear/logistic regression on $X$, whereas for \quadraticscenario we fit on a quadratic feature expansion of $X$.} We provide some additional results in the appendix where nuisances were instead fit via flexible neural networks, which show that this has little effect on our results. In all cases except for \deepgmm we perform optimization using LBFGS. Additional optimization details are given in the appendix.\footnote{Code for running all of our experiments is located at \url{https://github.com/CausalML/ESPRM}.}

For all configurations of scenario kind and policy we ran our experiments by sampling random scenarios of the respective kind, by setting all scenario parameters to be independent standard Gaussian variables. Specifically, for each $n \in \{100,200,500,1000,2000,5000,10000\}$ we sample 64 random scenarios of the respective kind, and for each random scenario we sample $n$ training data points and run all methods on this data. Results for \finitegmm, which generally did badly as predicted, are given in the appendix.

Define Relative Mean Regret Reduction (RMRR), given by:
\begin{equation*}
    RMRR(\hat\theta_n) = \left(1 - \frac{\e[\regretj(\hat\theta_n)]}{\e[\regretj(\hat\theta_n^{\text{ERM}]})]} \right)\times 100\%,
\end{equation*}
where each expectation in the fraction is taken over the joint distribution of randomly sampled scenario, and corresponding random estimate $\hat\theta$. Then for each scenario kind and policy class, we plot predicted RMRR against number of training data based on our \deepgmm estimates in \cref{fig:regret-plots}.
We see that \deepgmm consistently obtains policies on average that are lower regret or on-par than those obtained by \entropylearning (typically with around 10\% to 20\% RMRR), with 95\% confidence intervals indicating clearly better performance in every case except for training flexible policies on random \quadraticscenario scenarios (in which case performance seems roughly on par). It is notable that this even occurs in the \quadraticscenario setting with the linear policy class, where our policy class is not even well specified for the loss, let alone the surrogate loss.
We can also observe that the most significant regret benefits tend to occur with smaller training set sizes (since the same RMRR implies a larger absolute decrease in regret), indicating that the statistical efficiency of our method is leading to improved finite sample behavior.

\begin{figure}
\begin{center}
\ifpreprint
\begin{tabular}{c}
     \includegraphics[width=0.95\textwidth]{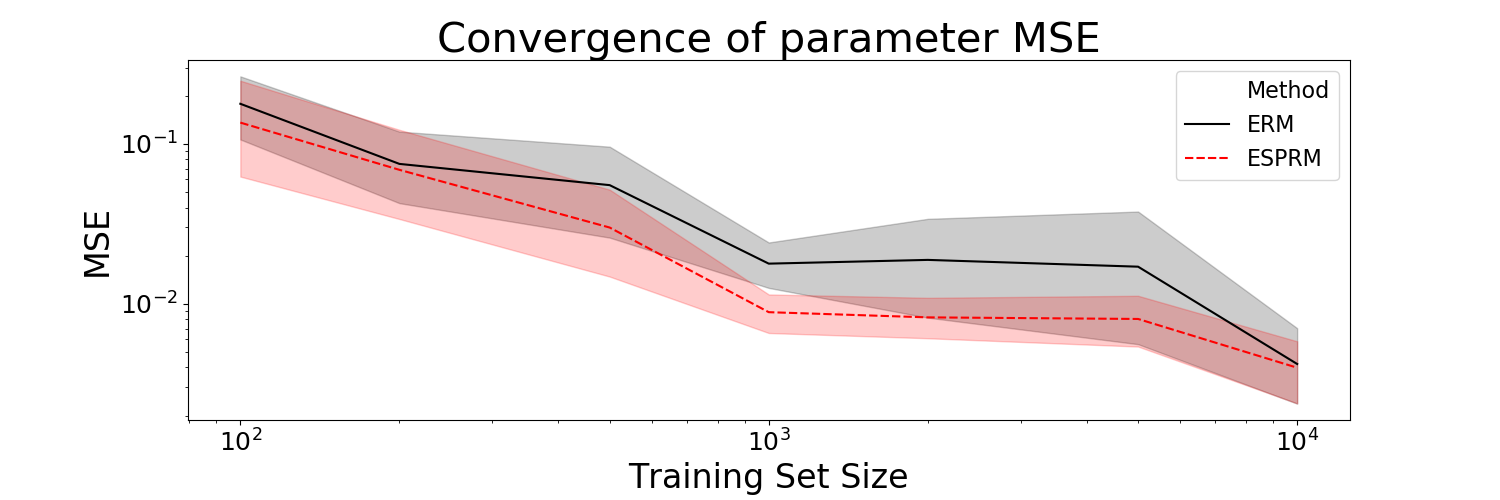}  \\ 
     \includegraphics[width=0.95\textwidth]{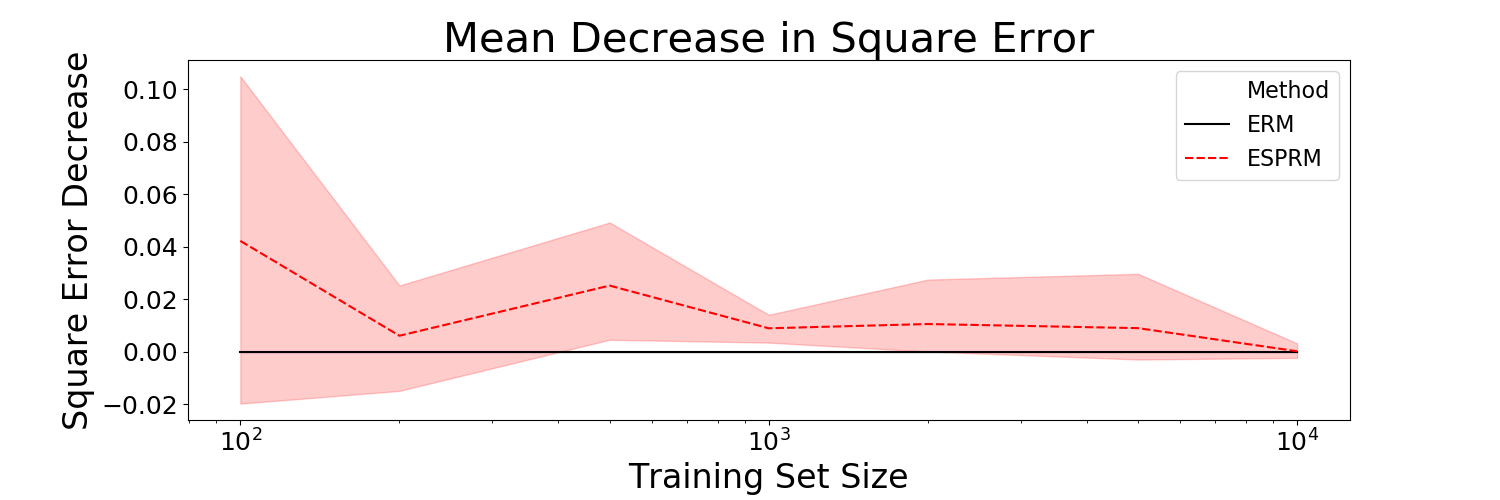} \\ 
\end{tabular}
\else
\begin{tabular}{c}
     \includegraphics[width=0.45\textwidth]{"figs/theta_mse_simple_linear"}  \\ \includegraphics[width=0.45\textwidth]{"figs/theta_mse_decraese_simple_linear"} \\ 
\end{tabular}
\fi
\end{center}
\caption{Above we plot the convergence in MSE of the predicted $\hat\theta_n$ for each method with a linear policy class, over the random scenarios of the \linearscenario class. Below we plot the average difference in the squared error of \deepgmm and \entropylearning (positive numbers indicate improvement over \entropylearning). All shaded regions are 95\% confidence intervals constructed from bootstrapping.}
\label{fig:mse-plots}
\end{figure}

In \cref{fig:mse-plots} we plot the convergence in terms of the MSE of the estimated parameter from \deepgmm and \entropylearning, for the \linearscenario setting and linear policy class (where parameters are low-dimensional and correctly specified). We plot both the MSE convergence, and the average difference in the squared error between the estimates, across the random scenarios.\footnote{All parameter vectors are normalized first given that the policy function is scale-invariant.} It is clear from these results that \deepgmm consistently estimates optimal policy parameters with lower squared error on average compared to \entropylearning across these random simulated scenarios. This provides strong evidence that the methodology indeed provides an improvement in statistical efficiency for solving the smooth surrogate loss problem.

\subsection{Jobs Case Study}

\begin{table}[t]
    \centering
    \begin{tabular}{lccc}
    \hline
    \textbf{Policy} & \textbf{\entropylearning} & \textbf{\deepgmm} & \textbf{Difference} \\
    \hline
    Linear & $-0.96 \pm 4.32$ & $4.42 \pm 3.78$ & $5.38 \pm 5.06$ \\
    Flexible & $-1.75 \pm 4.64$ & $7.68 \pm 3.16$ & $9.42 \pm 5.17$ \\
    \hline
    \end{tabular}
    \caption{Average predicted policy value (multiplied by 1000) for the Jobs case study for \entropylearning versus \deepgmm over 64 repetitions. The $\pm$ interval provides the 95\% confidence intervals.}
    \label{tab:job-results}
\end{table}

We next consider an application to a dataset derived from a large scale experiment comparing different programs offered to unemployed individuals in France \citep{behaghel2014private}.
We focus our attention to two arms from the experiment: a treatment arm where individuals receive an intensive counseling program run by a public agency and a treatment arm with a similar program run by a private agency. The hypothetical application is learning a personalized allocation to counseling program, with the aim of maximizing the number of individuals who reenter employment within six months, minus costs. (The original study's focus was not personalization.)
Our intervention is simply the offer of the counseling program; we therefore ignore the fact that some individuals offered one of the programs did not attend.

To make our policies focus on heterogeneous effects, we set the costs of each arm to be equal to their within-arm average outcome in the original data.
That is, the outcome we consider is equal whether one reentered employment within 6 months, minus the average number of individuals who entered employment within 6 months in that arm. 
The covariates we consider personalizing on are:
statistical risk of long-term unemployment,
whether individual is seeking full-time employment,
whether individual lives in sensitive suburban area,
whether individual has a college education,
the number of years of experience in the desired job, and
the nature of the desired job (\eg, technician, skilled clerical worker, etc.).

We then consider 64 replications of the following procedure.
Each time, we randomly split the data 40\%/60\% into train/test.
We then introduce some confounding into the training dataset. We consider the following three binary variables: whether individual has 1--5 years experience in the desired job, whether they seek a skilled blue collar job, and whether their statistical risk of long-term unemployment is medium. After studentizing each variable, we segment the data by the tertiles of their sum. In the first tertile, we drop each unit with probability $7/8$. In the second tertile, we drop private-program units with probability $1/4$ and public-program units with probability $7/8$. In the third tertile, we drop public-program units with probability $1/4$ and private-program units with probability $7/8$. Given a policy learned on this training data, we evaluate it on the held-out test set using a Horvitz-Thompson estimator.

Of the training data, $20\%$ was set aside for training nuisances, and an additional $20\%$ as validation data for early stopping.
We then trained both linear and flexible policies using \entropylearning and \deepgmm  as in our simulation studies, with the exception that nuisances were fitted using neural networks (of the same architecture as the flexible policy class).

We summarize the mean estimated outcome for the policies from each method in \cref{tab:job-results}.
We note from these values that on average \deepgmm seems to be learning higher value job-assignment policies than \entropylearning. Furthermore, performing paired two-sided $t$-tests on the two sets of repetitions for each policy to test for difference in mean policy value we obtained $p$-values of $.0429$ for the linear policy class and $.0007$ for the flexible policy class, clearly highlighting the benefit of our \deepgmm method.

\section{Conclusion}

We considered a common reduction of learning individualized treatment rules from observational data to weighted surrogate risk minimization. We showed that, quite differently from actual classification problems, assuming correct specification in the policy learning case actually suggests more efficient solutions to this reduction. 
In particular, even if we use efficient policy evaluation, this may not necessarily lead to efficient policy learning.
Specifically, under correct specification, the problem becomes a conditional moment problem in a semiparametric model and
efficiency here translates to both better MSE in estimating optimal policy parameters and improved regret bounds.

Based on this observation, we proposed an algorithm, \deepgmm, for efficiently solving the surrogate loss problem. We showed that our method consistently outperformed the standard method of empirical risk minimization on the surrogate loss, both over a wide variety of synthetic scenarios and in a case study based on a real job training experiment.

\clearpage

\bibliography{ref}

\begin{thebibliography}{28}
\providecommand{\natexlab}[1]{#1}
\providecommand{\url}[1]{\texttt{#1}}
\expandafter\ifx\csname urlstyle\endcsname\relax
  \providecommand{\doi}[1]{doi: #1}\else
  \providecommand{\doi}{doi: \begingroup \urlstyle{rm}\Url}\fi

\bibitem[Athey \& Wager(2017)Athey and Wager]{athey2017efficient}
Athey, S. and Wager, S.
\newblock Efficient policy learning.
\newblock \emph{arXiv preprint arXiv:1702.02896}, 2017.

\bibitem[Bartlett et~al.(2006)Bartlett, Jordan, and
  McAuliffe]{bartlett2006convexity}
Bartlett, P.~L., Jordan, M.~I., and McAuliffe, J.~D.
\newblock Convexity, classification, and risk bounds.
\newblock \emph{Journal of the American Statistical Association}, 101\penalty0
  (473):\penalty0 138--156, 2006.

\bibitem[Behaghel et~al.(2014)Behaghel, Cr{\'e}pon, and
  Gurgand]{behaghel2014private}
Behaghel, L., Cr{\'e}pon, B., and Gurgand, M.
\newblock Private and public provision of counseling to job seekers: Evidence
  from a large controlled experiment.
\newblock \emph{American economic journal: applied economics}, 6\penalty0
  (4):\penalty0 142--74, 2014.

\bibitem[Bennett et~al.(2019)Bennett, Kallus, and Schnabel]{bennett2019deep}
Bennett, A., Kallus, N., and Schnabel, T.
\newblock Deep generalized method of moments for instrumental variable
  analysis.
\newblock In \emph{Advances in Neural Information Processing Systems}, pp.\
  3559--3569, 2019.

\bibitem[Beygelzimer \& Langford(2009)Beygelzimer and
  Langford]{beygelzimer2009offset}
Beygelzimer, A. and Langford, J.
\newblock The offset tree for learning with partial labels.
\newblock In \emph{Proceedings of the 15th ACM SIGKDD international conference
  on Knowledge discovery and data mining}, pp.\  129--138, 2009.

\bibitem[Carrasco \& Florens(2014)Carrasco and Florens]{carrasco2014asymptotic}
Carrasco, M. and Florens, J.-P.
\newblock On the asymptotic efficiency of gmm.
\newblock \emph{Econometric Theory}, 30\penalty0 (2):\penalty0 372--406, 2014.

\bibitem[Chernozhukov et~al.(2018)Chernozhukov, Chetverikov, Demirer, Duflo,
  Hansen, Newey, and Robins]{chernozhukov2017double}
Chernozhukov, V., Chetverikov, D., Demirer, M., Duflo, E., Hansen, C., Newey,
  W., and Robins, J.
\newblock Double/debiased machine learning for treatment and structural
  parameters.
\newblock \emph{The Econometrics Journal}, 21\penalty0 (1):\penalty0 C1--C68,
  2018.

\bibitem[Chernozhukov et~al.(2019)Chernozhukov, Demirer, Lewis, and
  Syrgkanis]{chernozhukov2019semi}
Chernozhukov, V., Demirer, M., Lewis, G., and Syrgkanis, V.
\newblock Semi-parametric efficient policy learning with continuous actions.
\newblock In \emph{Advances in Neural Information Processing Systems}, pp.\
  15039--15049, 2019.

\bibitem[Daskalakis et~al.(2017)Daskalakis, Ilyas, Syrgkanis, and
  Zeng]{daskalakis2017training}
Daskalakis, C., Ilyas, A., Syrgkanis, V., and Zeng, H.
\newblock Training gans with optimism.
\newblock \emph{arXiv preprint arXiv:1711.00141}, 2017.

\bibitem[Dud{\'\i}k et~al.(2011)Dud{\'\i}k, Langford, and Li]{dudik2011doubly}
Dud{\'\i}k, M., Langford, J., and Li, L.
\newblock Doubly robust policy evaluation and learning.
\newblock \emph{arXiv preprint arXiv:1103.4601}, 2011.

\bibitem[Hahn(1998)]{hahn1998role}
Hahn, J.
\newblock On the role of the propensity score in efficient semiparametric
  estimation of average treatment effects.
\newblock \emph{Econometrica}, pp.\  315--331, 1998.

\bibitem[Hansen(1982)]{hansen1982large}
Hansen, L.~P.
\newblock Large sample properties of generalized method of moments estimators.
\newblock \emph{Econometrica}, pp.\  1029--1054, 1982.

\bibitem[Hirano et~al.(2003)Hirano, Imbens, and Ridder]{hirano2003efficient}
Hirano, K., Imbens, G.~W., and Ridder, G.
\newblock Efficient estimation of average treatment effects using the estimated
  propensity score.
\newblock \emph{Econometrica}, 71\penalty0 (4):\penalty0 1161--1189, 2003.

\bibitem[Jiang et~al.(2019)Jiang, Song, Li, and Zeng]{jiang2019entropy}
Jiang, B., Song, R., Li, J., and Zeng, D.
\newblock Entropy learning for dynamic treatment regimes.
\newblock \emph{Statistica Sinica}, 2019.

\bibitem[Kallus(2017)]{kallus2017recursive}
Kallus, N.
\newblock Recursive partitioning for personalization using observational data.
\newblock In \emph{Proceedings of the 34th International Conference on Machine
  Learning}, pp.\  1789--1798, 2017.

\bibitem[Kallus(2018)]{kallus2018balanced}
Kallus, N.
\newblock Balanced policy evaluation and learning.
\newblock In \emph{Advances in Neural Information Processing Systems}, pp.\
  8895--8906, 2018.

\bibitem[Kallus \& Zhou(2018)Kallus and Zhou]{kallus2018policy}
Kallus, N. and Zhou, A.
\newblock Policy evaluation and optimization with continuous treatments.
\newblock \emph{arXiv preprint arXiv:1802.06037}, 2018.

\bibitem[Khosravi et~al.(2019)Khosravi, Lewis, and Syrgkanis]{khosravi2019non}
Khosravi, K., Lewis, G., and Syrgkanis, V.
\newblock Non-parametric inference adaptive to intrinsic dimension.
\newblock \emph{arXiv preprint arXiv:1901.03719}, 2019.

\bibitem[Kitagawa \& Tetenov(2018)Kitagawa and Tetenov]{kitagawa2018should}
Kitagawa, T. and Tetenov, A.
\newblock Who should be treated? empirical welfare maximization methods for
  treatment choice.
\newblock \emph{Econometrica}, 86\penalty0 (2):\penalty0 591--616, 2018.

\bibitem[Krishnamurthy et~al.(2019)Krishnamurthy, Langford, Slivkins, and
  Zhang]{krishnamurthy2019contextual}
Krishnamurthy, A., Langford, J., Slivkins, A., and Zhang, C.
\newblock Contextual bandits with continuous actions: Smoothing, zooming, and
  adapting.
\newblock \emph{arXiv preprint arXiv:1902.01520}, 2019.

\bibitem[Newey(1993)]{newey1993efficient}
Newey, W.~K.
\newblock Efficient estimation of models with conditional moment restrictions.
\newblock \emph{Handbook of Statistics}, 11:\penalty0 419--454, 1993.

\bibitem[Qian \& Murphy(2011)Qian and Murphy]{qian2011performance}
Qian, M. and Murphy, S.~A.
\newblock Performance guarantees for individualized treatment rules.
\newblock \emph{Annals of statistics}, 39\penalty0 (2):\penalty0 1180, 2011.

\bibitem[Rahimi \& Recht(2009)Rahimi and Recht]{rahimi2009weighted}
Rahimi, A. and Recht, B.
\newblock Weighted sums of random kitchen sinks: Replacing minimization with
  randomization in learning.
\newblock In \emph{Advances in neural information processing systems}, pp.\
  1313--1320, 2009.

\bibitem[Swaminathan \& Joachims(2015)Swaminathan and
  Joachims]{swaminathan2015counterfactual}
Swaminathan, A. and Joachims, T.
\newblock Counterfactual risk minimization: Learning from logged bandit
  feedback.
\newblock In \emph{International Conference on Machine Learning}, pp.\
  814--823, 2015.

\bibitem[Van~der Vaart(2000)]{van2000asymptotic}
Van~der Vaart, A.~W.
\newblock \emph{Asymptotic statistics}.
\newblock Cambridge University Press, 2000.

\bibitem[Zhao et~al.(2012)Zhao, Zeng, Rush, and Kosorok]{zhao2012estimating}
Zhao, Y., Zeng, D., Rush, A.~J., and Kosorok, M.~R.
\newblock Estimating individualized treatment rules using outcome weighted
  learning.
\newblock \emph{Journal of the American Statistical Association}, 107\penalty0
  (499):\penalty0 1106--1118, 2012.

\bibitem[Zhou et~al.(2017)Zhou, Mayer-Hamblett, Khan, and
  Kosorok]{zhou2017residual}
Zhou, X., Mayer-Hamblett, N., Khan, U., and Kosorok, M.~R.
\newblock Residual weighted learning for estimating individualized treatment
  rules.
\newblock \emph{Journal of the American Statistical Association}, 112\penalty0
  (517):\penalty0 169--187, 2017.

\bibitem[Zhou et~al.(2018)Zhou, Athey, and Wager]{zhou2018offline}
Zhou, Z., Athey, S., and Wager, S.
\newblock Offline multi-action policy learning: Generalization and
  optimization.
\newblock \emph{arXiv preprint arXiv:1810.04778}, 2018.

\end{thebibliography}
\bibliographystyle{icml2020}

\clearpage
\onecolumn
\appendix

\section{Omitted Proofs}

To prove \cref{thm:correct-specification,thm:conditoinalmomentproblem}, we first establish the following two useful lemmas.

Define 
\begin{align*}
\overline{\mathcal G}^*&={\argmin_{g~\text{unconstrained}}\e[\abs{\psi} l(g(X), \sign(\psi))]},\\
\mathcal G^*&={\argmin_{g\in\mathcal G}\e[\abs{\psi} l(g(X), \sign(\psi))]}.
\end{align*}
Correct specification, \cref{eq:specification}, is the assumption that $\mathcal G\cap \overline {\mathcal G}^*\neq\varnothing$.

\begin{lemma}\label{lemma:specification}
Suppose $\mathcal G\cap \overline {\mathcal G}^*\neq\varnothing$. Then $$\mathcal G\cap \overline{\mathcal G}^*={\mathcal G}^*.$$
\end{lemma}
\begin{proof}
For brevity, let $c(g)=\e[\abs{\psi} l(g(X), \sign(\psi))]$.

Let any $g\in\mathcal G\cap \overline{\mathcal G}^*$ be given. Now let any $g'\in \mathcal G$ be given. Since $g\in\overline{\mathcal G}^*$, we have $c(g)\leq c(g')$. Since $g'\in\mathcal G$ was arbitrary, we conclude that $g\in\mathcal G^*$.

Now, let any $g\in\mathcal G^*$ be given. By assumption, $\exists g^*\in\mathcal G\cap \overline {\mathcal G}^*$. Since $g\in\mathcal G^*$ and $g^*\in\mathcal G$, we obtain that $c(g)\leq c(g^*)$. Now let any $g'$ unconstrained be given. Since $g^*\in\overline {\mathcal G}^*$, we have $c(g^*)\leq c(g')$, whence $c(g)\leq c(g')$. Since $g'$ unconstrained was arbitrary, we conclude that $g\in\overline{\mathcal G}^*$. Since $g\in\mathcal G$ by definition of $\mathcal G^*$, we conclude that $g\in\mathcal G\cap \overline{\mathcal G}^*$.
\end{proof}

\begin{lemma}\label{lemma:firstorder}
Suppose $\e[\abs{\psi}\mid X]<\infty$ almost surely. Then
$$
\overline{\mathcal G}^*=\{g(\cdot):\e\left[\abs{\psi}l'(g(X), \sign(\psi))\mid X\right] = 0~\text{almost surely}\}.
$$
\end{lemma}
\begin{proof}
Notice that because $g$ is an unconstrained function of $X$ it must minimize the conditional expectation. That is,
$$
\overline{\mathcal G}^*=\{g:g(x)\in\argmin_{z\in\mathbb R}\e[\abs{\psi} l(z, \sign(\psi))\mid X=x]~\text{for a.e. $x$}\}.
$$
Since $\abs{\psi} l(z, \sign(\psi))$ is convex in $z$, so is $\e[\abs{\psi} l(z, \sign(\psi))\mid X=x]$.
Next, note that by mean value theorem, we have
\begin{align*}
&\frac{\partial}{\partial z}\e[\abs{\psi} l(z, \sign(\psi))\mid X=x] \\
&\qquad=\lim_{h\to 0}
\e\left[\abs{\psi}\frac{l(z+h, \sign(\psi))-\abs{\psi} l(z, \sign(\psi))}h\mid X=x\right]
\\&\qquad=\lim_{h\to 0}
\e\left[\abs{\psi}l'(z(h), \sign(\psi))\mid X=x\right]
\end{align*}
for some $z(h)\in[z,z+h]$. Since $\abs{l'(z(h), \sign(\psi))}\leq 4$ and $\e\left[\abs{\psi}\mid X=x\right]<\infty$, dominated convergence theorem yields
\begin{align*}
\frac{\partial}{\partial z}\e[\abs{\psi} l(z, \sign(\psi))\mid X=x]
&=
\e\left[\abs{\psi}\lim_{h\to 0}l'(z(h), \sign(\psi))\mid X=x\right]
\\&=
\e\left[\abs{\psi}l'(z, \sign(\psi))\mid X=x\right].
\end{align*}
We conclude via first-order conditions for unconstrained optimization over $z$ that
$$\overline{\mathcal G}^*=\{g:\e\left[\abs{\psi}l'(g(x), \sign(\psi))\mid X=x\right]=0~\text{for a.e. $x$}\},$$
which is a restatement of the lemma's result.
\end{proof}

We are now prepared to prove \cref{thm:correct-specification,thm:conditoinalmomentproblem}.

\begin{proof}[Proof of \cref{thm:correct-specification}]
Suppose $\theta^*\in\argmin_{\theta\in\Theta}L(\theta)$. 
That is,
$g_{\theta^*}\in\mathcal G^*$. By \cref{lemma:specification},
$g_{\theta^*}\in\overline{\mathcal G}^*$. Then, by \cref{lemma:firstorder},
$\e\left[\abs{\psi}l'(g_{\theta^*}(x), \sign(\psi))\mid X=x\right]=0$ for a.e. $x$. Consider such an $x$.
Note that $l'(g,s)=-s\sigma(-sg)$, and define $a_x = \e[\abs{\psi} \indicator{\psi \geq 0} \mid X=x]$ and $b_x = \e[\abs{\psi} \indicator{\psi < 0} \mid X=x]$.
\begin{align*}
0&=\e\left[\psi\sigma(-\sign(\psi)g_{\theta^*}(x))\mid X=x\right]\\
 &=\sigma(-g_{\theta^*}(x))a_x - \sigma(g_{\theta^*}(x)) b_x \\
 &= a_x - \sigma(g_{\theta^*}(x))(a_x + b_x),
\end{align*}
hence
\begin{align*}
\sigma(g_{\theta^*}(x)) &= \frac{1}{1 + b_x / a_x} \\
g_{\theta^*}(x) &= \log\left(\frac{a_x}{b_x}\right).
\end{align*}
We therefore have, from $\sign(\log(a/b))=\sign(a-b)$,
\begin{align*}
\sign(g_{\theta^*}(x))&=\sign({\e\left[\abs{\psi}\indicator{\psi\geq0}\mid X=x\right]}-{\e\left[\abs{\psi}\indicator{\psi<0}\mid X=x\right]})\\
&=\sign(\e[\psi\mid X=x]).
\end{align*}
The condition that $\sign(g(x))=\sign(\e[\psi\mid X=x])$ for almost every $x$ is exactly equivalent to the condition that $\sign(g(\cdot))\in\max_{\pi~\text{unconstrained}}J(\pi)$ since, by assumption on $\psi$ and iterated expectations, we have that $J(\pi)=\e[\pi(X)(Y(1)-Y(-1))]=\e[\pi(X)\e[Y(1)-Y(-1)\mid X]]=\e[\pi(X)\psi]$.
\end{proof}

\begin{proof}[Proof of \cref{thm:conditoinalmomentproblem}]
First note that \cref{asm:psi-bounded} implies that $\e[\abs{\psi} \mid X=x] < \infty$ almost everywhere, so the conditions of \cref{lemma:firstorder} apply.
Now suppose $\theta^*\in\argmin_{\theta\in\Theta}L(\theta)$. 
By \cref{lemma:specification},
$g_{\theta^*}\in\overline{\mathcal G}^*$. Then, by \cref{lemma:firstorder},
$\e\left[\abs{\psi}l'(g_{\theta^*}(x), \sign(\psi))\mid X=x\right]=0$ for a.e. $x$, which is a restatement of $m(X;\theta^*)=0$ almost surely.

Now suppose $m(X;\theta^*)=0$ almost surely. Then, by \cref{lemma:firstorder}, $g_{\theta^*}\in\overline{\mathcal G}^*$. By definition, $g_{\theta^*}\in{\mathcal G}$. Therefore, by \cref{lemma:specification}, $g_{\theta^*}\in{\mathcal G^*}$, which is a restatement of $\theta^*\in\argmin_{\theta\in\Theta}L(\theta)$.
\end{proof}

\begin{proof}[Proof of \cref{lemma:semimodel}]
Suppose that $\Pi$ is correctly specified for the surrogate loss. Then given \cref{asm:psi-bounded}, there exists $\theta^*$ such that, for each $x$ almost everywhere, we have:
\begin{align*}
    &\e[\abs{\hat\psi} (2 \sigma(g_{\theta^*}(X)) - (\sign(\hat\psi) + 1))  \mid X = x] = 0 \\
    &\iff \e[\abs{\hat\psi} (2 \sigma(g_{\theta^*}(x)) - (\sign(\hat\psi) + 1)) \mid X = x] = 0 \\
    &\iff 2 \sigma(g_{\theta^*}(x)) \e[\abs{\hat\psi} \mid X=x] = 2 P(\hat\psi > 0 \mid X=x) \\
    &\iff \sigma(g_{\theta^*}(x)) = \frac{ P(\hat\psi > 0 \mid X=x) }{ \e[\abs{\hat\psi} \mid X = x] }.
\end{align*}
\end{proof}

\begin{proof}[Proof of \cref{thm:regret-bound}]
Let $S = \sign(\psi)$. We first note that given $\e[\abs{\psi}] < \infty$ from \cref{asm:psi-bounded}, we have that $J$ and $L$ can be re-scaled by a factor of $\e[\abs{\psi}]$ and expressed as the expected values of $\indicator{S = \sign(g_\theta(X))}$ and $\phi(S g_{\theta}(X))$, where $\phi(\alpha) = 2\log(1 + \exp(\alpha)) - 2\alpha$, for some modified distribution of $X,S$ (where the measure $\mu(x,s)$ of $X,S$ is re-scaled by $\e[\abs{\psi} \mid X=x,S=s]$). In addition by our correct specification assumption we have $\regretl(\theta) = L(\theta) - L^*$, where $L^*$ is the minimum loss over all possible unconstrained choices of function $g$. Given the above it follows from \citet[Theorem 3]{bartlett2006convexity} that $w(\regretj) \leq \regretl$ for some non-decreasing, non-negative function $w: [0,1] \to [0,\infty)$, which depends only on the nonnegative loss function $\phi$. Following their notation, define:
\begin{align*}
    H(\eta) &= \inf_{\alpha \in \mathbb R} \eta \phi(\alpha) + (1 - \eta) \phi(-\alpha) \\
    H^-(\eta) &= \inf_{\alpha : \alpha(2 \eta - 1) \leq 0} \eta \phi(\alpha) + (1 - \eta) \phi(-\alpha) \\
    \tilde w(\theta) &= H^-\left(\frac{1+\theta}{2}\right) - H\left(\frac{1+\theta}{2}\right).
\end{align*}
Then it follows from \citet[Section 2]{bartlett2006convexity} that $w$ is the Fenchel-Legendre biconjugate of $\tilde w$. Now it is easy to verify from these definitions that $\tilde w(\theta) = \abs{\theta}$, which is convex, and thus $w(\theta) = \tilde w(\theta) = \abs{\theta}$. The desired result follows immediately from this since $\abs{\regretj} = \regretj$.
\end{proof}

\begin{proof}[Proof of \cref{thm:optimal-regret}]
Given \cref{asm:regular-loss}, from the Taylor expansion from \cref{sec:optimalregret} we have:
\begin{equation*}
    n \regretl(\hat\theta_n) = (\sqrt{n} (\hat\theta_n - \theta^*))^T H(\theta^*) (\sqrt{n} (\hat\theta_n - \theta^*)) + \norm{\sqrt{n}(\hat\theta_n - \theta^*)}^2 o(1)
\end{equation*}

Thus assuming that $\hat\theta_n$ is regular, we let $W$ be the limiting distribution of $\sqrt{n}(\hat\theta_n - \theta^*)$, which by Slutsky's and the continuous mapping theorem gives us that
\begin{equation*}
    \asymregretl(\hat\theta_n) = W^T H(\theta^*) W.
\end{equation*}

Now, by \citet[Theorem 25.20]{van2000asymptotic} we have that $W = N(0, V) * M$, where $M$ is given by some arbitrary distribution, $V$ is the convariance matrix of the semi-parametrically efficient estimator, and $*$ denotes convolution. In addition $M = 0 \ \text{a.s.} \iff \hat\theta_n$ is semi-parametrically efficient. Now let $W^* = N(0, V)$. Then it follows from \citet[Lemma 8.5]{van2000asymptotic} that $\e[\phi(W^T H(\theta^*) W))] \geq \e[\phi(W^{*T} H(\theta^*) W^*)]$ for any $W = W^* * M$, since $l(w) = \phi(w^T H(\theta^*)w)$ is a bowl-shaped loss in the sense of \citet{van2000asymptotic} given that $\phi$ is non-negative and non-decreasing. Thus we can conclude by noting that the efficiency bound is given by $B_{\phi} = \e[\phi(W^{*T} H(\theta^*) W^*)]$, which is clearly realized for any semi-parametrically efficient $\hat\theta_n$.
\end{proof}

\section{Additional Experiment Details}

\subsection{Additional Optimization Details}

\paragraph{Solving \deepgmm Smooth Game} As mentioned in \cref{sec:deepgmm}, we solve the smooth game by running alternating first-order optimization using the OAdam algorithm. We tuned this procedure manually by experimenting on a couple of hand selected synthetic scenarios, one \linearscenario and one \quadraticscenario, prior to running our main experiments. We found generally good results using a learning rate of 0.001 for linear policy networks, and 0.0002 for flexible policy networks, with the learning rate of the critic $f$ network set to 5 times that of the policy netowrk. Furthermore we found good results using a number of epochs given by the fixed rule of $\min(8000000 / n, 8000)$, where $n$ is the number of training data points used.

\paragraph{Optimizing Neural Networks for Nuisance Functions and \finitegmm} In all cases where we optimized neural networks in these problems we used the LBFGS algorithm. Furthermore we performed some additional first-order optimization using Adam to deal with potential cases of poor convergence, using a learning rate of 0.001, and stopping once performance on a held-out validation set (of same size as training set) failed to improve for 5 consecutive epochs.

\subsection{Results for \finitegmm Methods}

We include here results for our \finitegmm method. As mentioned in \cref{sec:experiments}, the results for these methods were poor as expected. In particular the results seem to be very unstable, with extremely poor policy learning in a small percentage of cases, leading to extremely negative RMRR values in all cases except with \quadraticscenario scenario and linear policy network. However even in the majority of cases where these estimators don't have unstable behavior, they seem to perform par with or at best only marginally better than \entropylearning, with the one expection of \quadraticscenario scenario and linear policy network.

In our experiments with \finitegmm we experimented with two different kinds of choices for the set of critic functions $\mathcal F$: (1) polynomial expansion of $X$ of degree $d$; and (2) Random Kitchen Sink (RKS; \citet{rahimi2009weighted}) expansion of $X$ of length $n$ using the Gaussian kernel with $\sigma = 0.5$. Note that the Random Kitchen Sink expansion is designed such that $\phi_n(x_1)^T \phi_n(x_2) \approx K(x_1,x_2)$ for some given kernel, with approximation error vanishing as $n \rightarrow \infty$. In both cases, the function $f_i$ is given by the $i$'th coordinate of the corresponding feature map. We calculated $\hat\theta_n^{\finitegmm}$ using 3 stages, with the guess of $\tilde\theta_0$ in the first stage chosen at random.

In \cref{fig:fgmm-plots} we plot the performance of both \finitegmm and \deepgmm in terms of the RMRR metric, plotting both mean and median values across different values of $n$. Although we experimented with multiple choices of polynomial degree / RKS expansion length, we only plot results for degree 3 polynomials ($\text{Poly}(3)$), and length 64 expansions ($\text{RBF}(64)$) for clarity, as we found these gave the least-worst results.

\begin{figure*}
\begin{center}
\ifpreprint
\begin{tabular}{c}
     \includegraphics[width=0.95\textwidth]{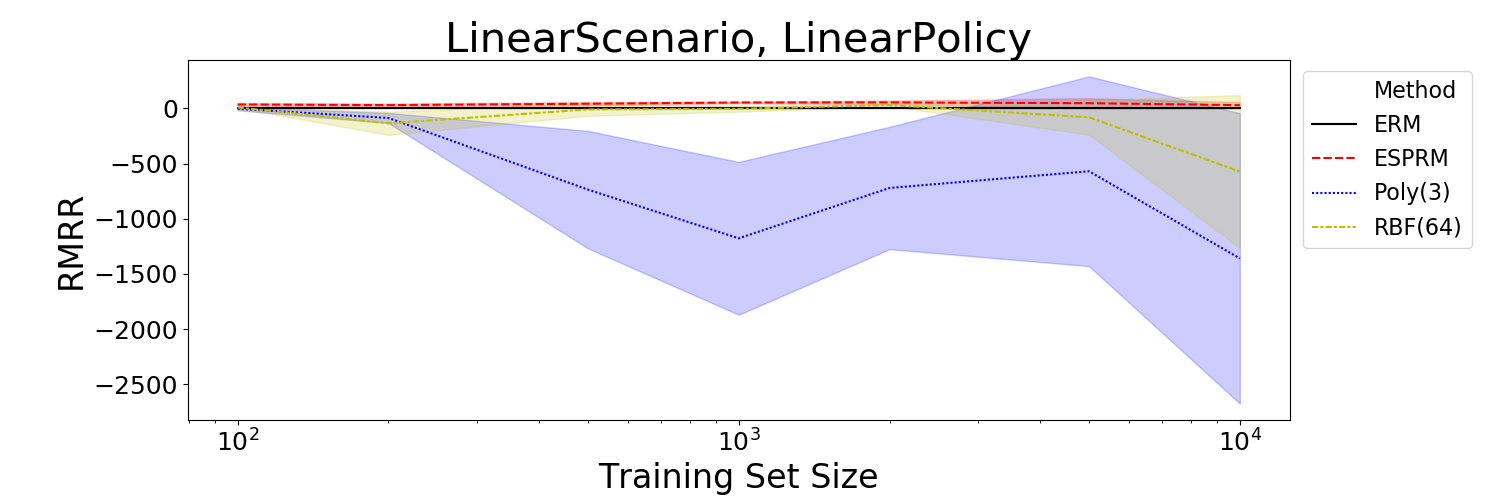} \\
     \includegraphics[width=0.95\textwidth]{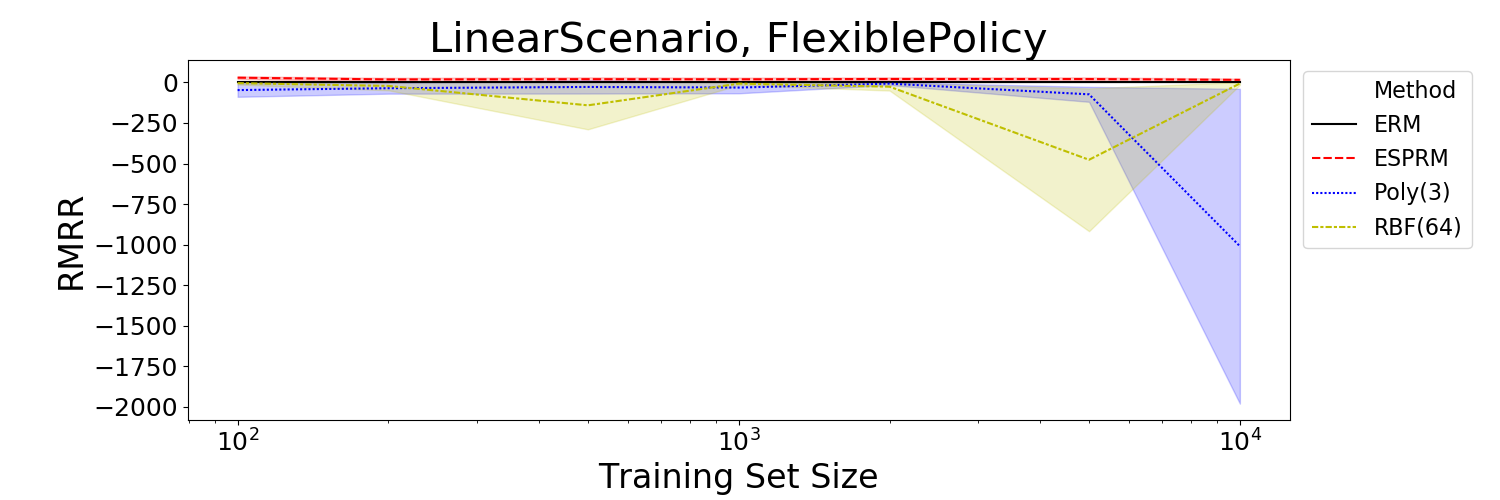} \\
     \includegraphics[width=0.95\textwidth]{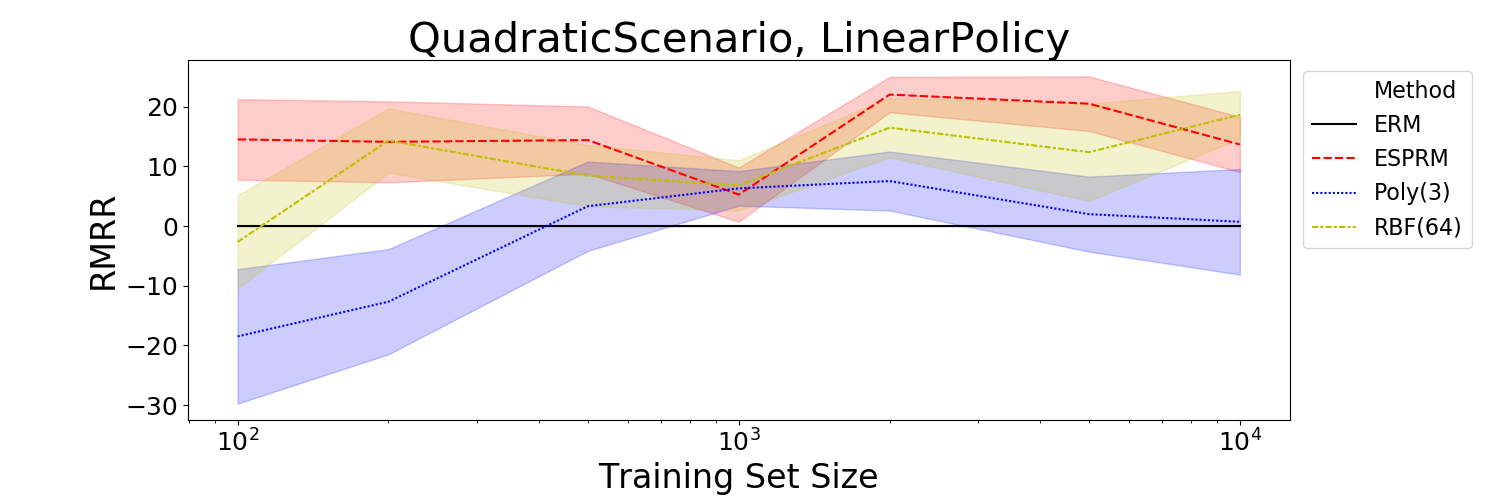} \\
     \includegraphics[width=0.95\textwidth]{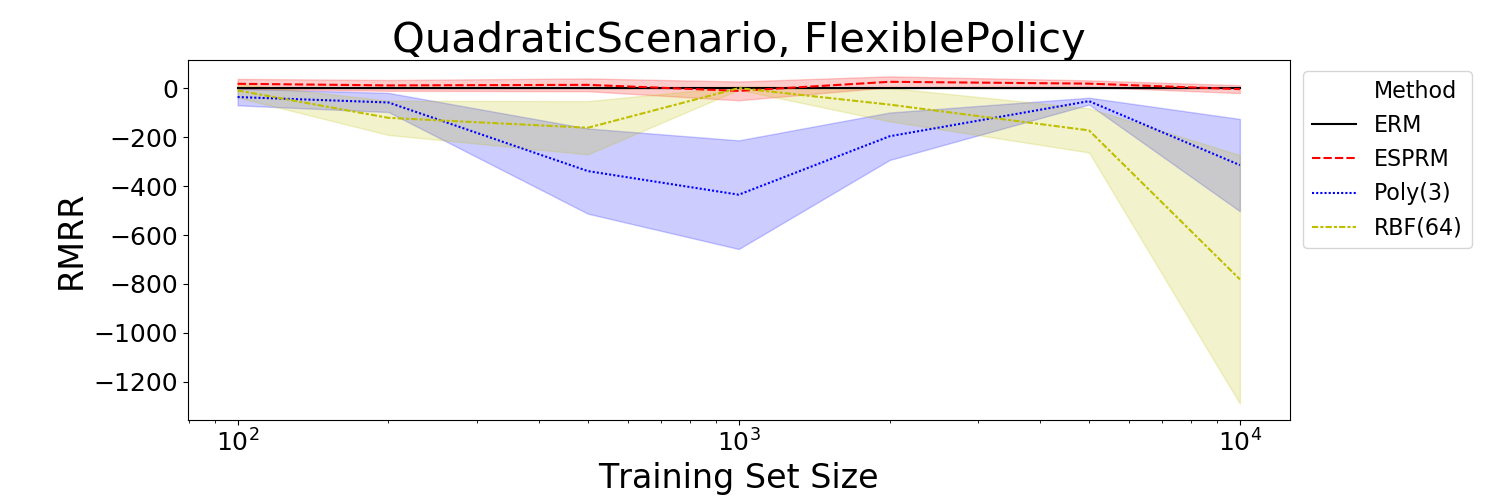}
\end{tabular}
\else
\begin{tabular}{cc}
     \includegraphics[width=0.45\textwidth]{"figs/fgmm_simple_linear"} & \includegraphics[width=0.45\textwidth]{"figs/fgmm_simple_flexible"} \\ 
     \includegraphics[width=0.45\textwidth]{"figs/fgmm_quadratic_linear"} & \includegraphics[width=0.45\textwidth]{"figs/fgmm_quadratic_flexible"}
\end{tabular}
\fi
\end{center}
\caption{Results for \deepgmm, \entropylearning, and \finitegmm methods for all scenarios and policy network types, where nuisances are fit using linear/logistic regression as in main experiments.}
\label{fig:fgmm-plots}
\end{figure*}

\subsection{Additional Results for Flexible Nuisance fitting}

We provide here some additional results for our simulation study on \quadraticscenario where the nuisances were fit using flexible neural network training (using the same neural network architecture as for the flexible policy class) instead of using a correctly specified model. We show results for the \deepgmm and \entropylearning methods in \cref{fig:nuisance-plots} both with nuisance fit using correctly specified model and using flexible neural network model. We note that results are about the same in both cases: for linear policies we have clearly superior results with our method versus \entropylearning, while for flexible policies in both cases the methods are roughly on par with each other, with slight performance increase in favor of \deepgmm for some values of $n$.

\begin{figure*}
\begin{center}
\ifpreprint
\begin{tabular}{c}
     \includegraphics[width=0.95\textwidth]{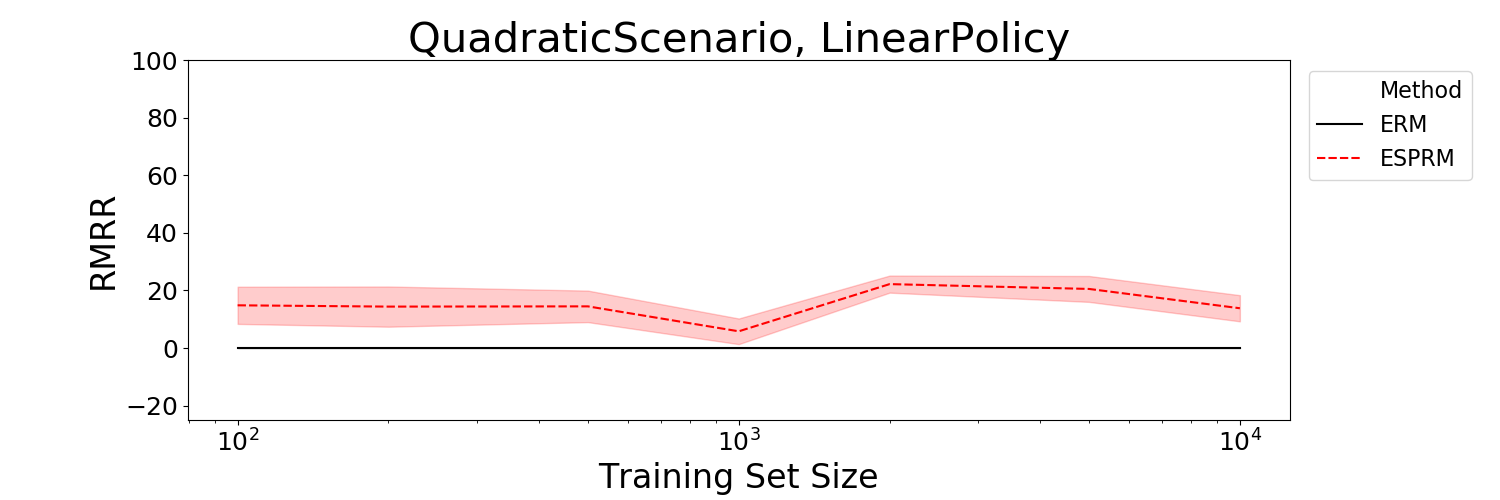} \\ 
     \includegraphics[width=0.95\textwidth]{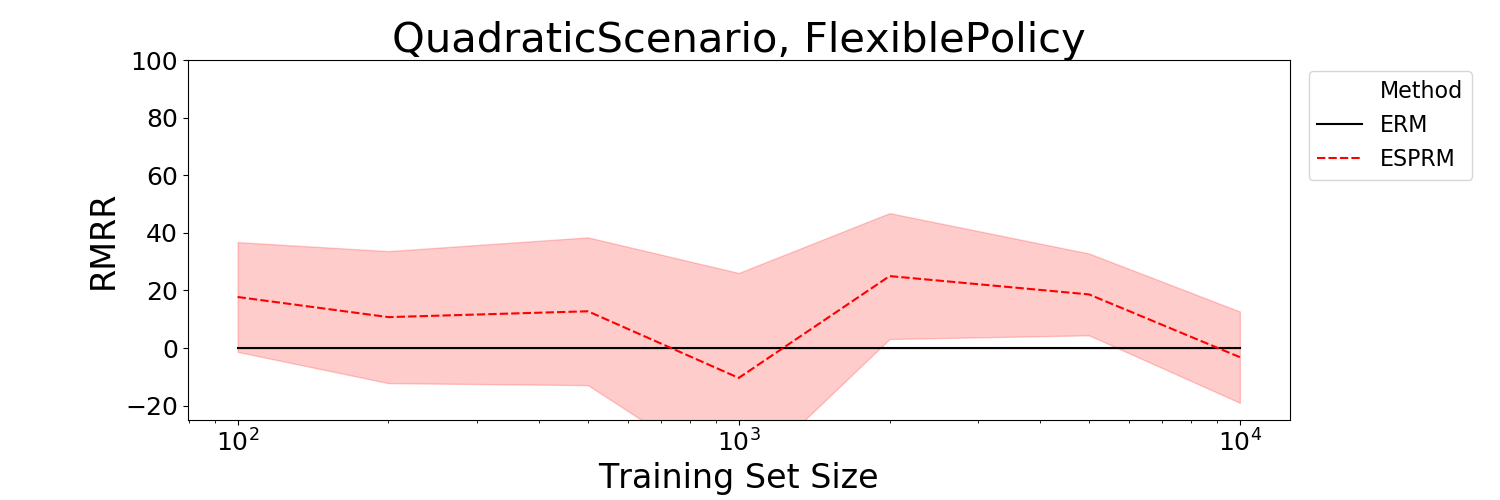} \\ 
     \includegraphics[width=0.95\textwidth]{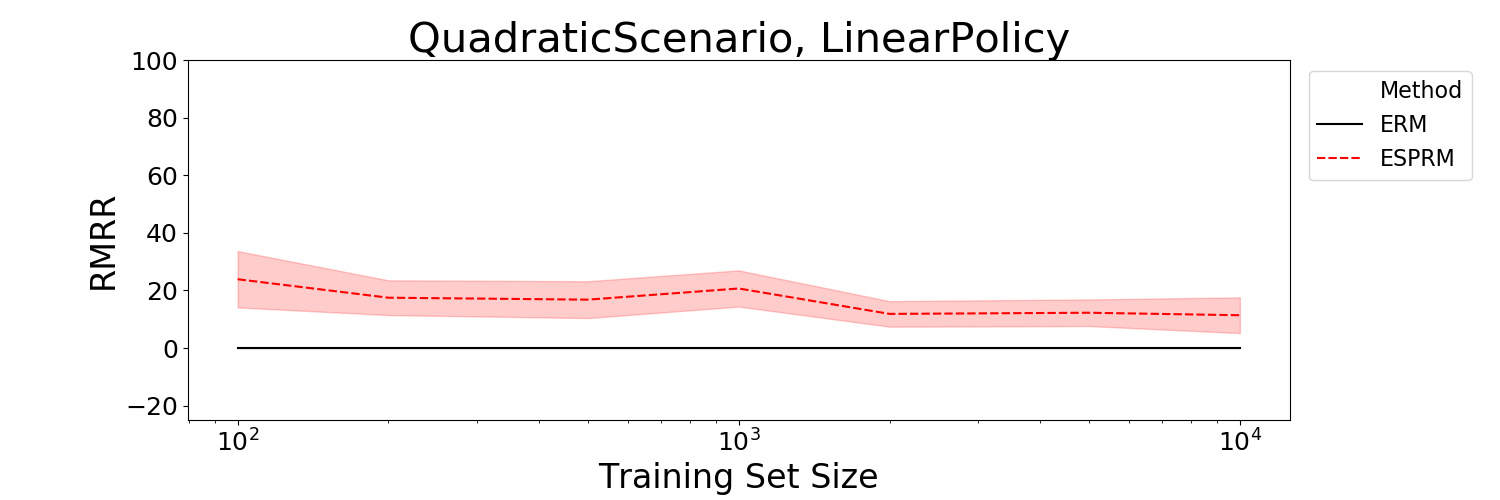} \\
     \includegraphics[width=0.95\textwidth]{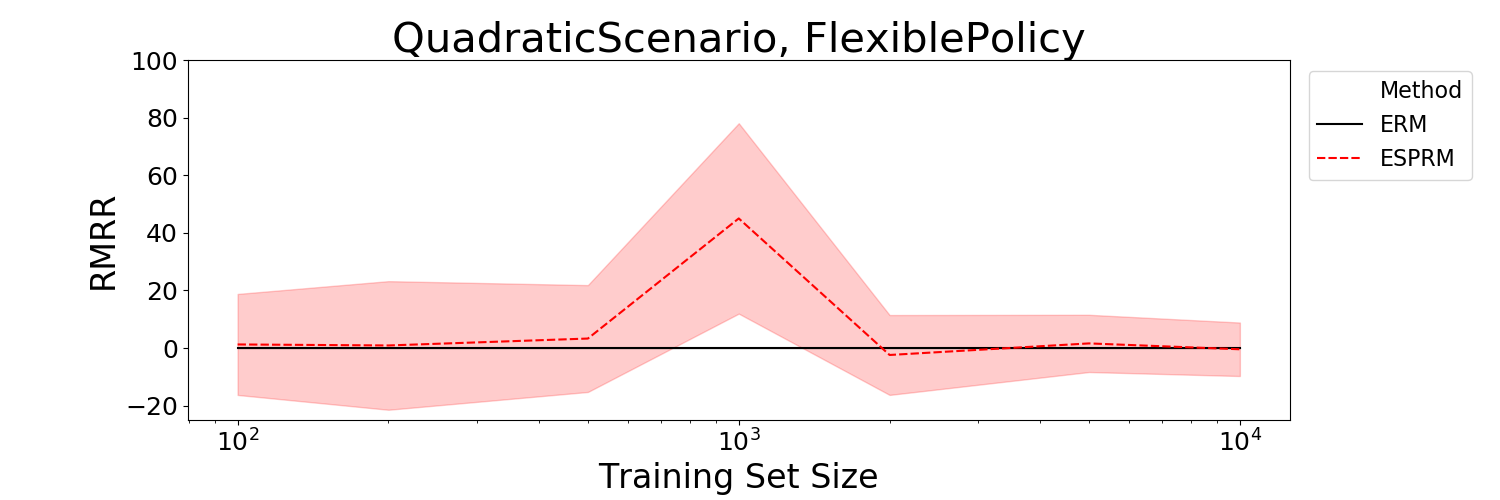}
\end{tabular}
\else
\begin{tabular}{cc}
     \includegraphics[width=0.45\textwidth]{"figs/apx_quadratic_linear"} & \includegraphics[width=0.45\textwidth]{"figs/apx_quadratic_flexible"} \\ 
     \includegraphics[width=0.45\textwidth]{"figs/apx_nuisance_quadratic_linear"} & \includegraphics[width=0.45\textwidth]{"figs/apx_nuisance_quadratic_flexible"}
\end{tabular}
\fi
\end{center}
\caption{Results for both \deepgmm and \entropylearning methods for \quadraticscenario, where in top row results obtained by fitting correctly specified nuisance model, while in bottom row results fit using flexible nueral network nuisance model.}
\label{fig:nuisance-plots}
\end{figure*}

\end{document}